\def\year{2022}\relax
\documentclass[letterpaper]{article} % DO NOT CHANGE THIS
\usepackage{aaai22}  % DO NOT CHANGE THIS
\usepackage{times}  % DO NOT CHANGE THIS
\usepackage{helvet}  % DO NOT CHANGE THIS
\usepackage{courier}  % DO NOT CHANGE THIS
\usepackage[hyphens]{url}  % DO NOT CHANGE THIS
\usepackage{graphicx} % DO NOT CHANGE THIS
\usepackage{natbib}  % DO NOT CHANGE THIS AND DO NOT ADD ANY OPTIONS TO IT
\usepackage{caption} % DO NOT CHANGE THIS AND DO NOT ADD ANY OPTIONS TO IT
\DeclareCaptionStyle{ruled}{labelfont=normalfont,labelsep=colon,strut=off} % DO NOT CHANGE THIS
\usepackage{newfloat}
\usepackage{listings}
\usepackage{booktabs}
\usepackage{multirow}
\usepackage{adjustbox}
\usepackage{booktabs}       % professional-quality tables
\usepackage{amsfonts}       % blackboard math symbols

\usepackage{amsmath}
\usepackage{amsthm}
\usepackage{mathtools}

\DeclareMathOperator*{\argmin}{arg\,min}
\usepackage[linesnumbered,vlined, ruled]{algorithm2e}
\renewcommand{\KwData}{\textbf{Input:}}
\renewcommand{\KwResult}{\textbf{Output:}}

\newtheorem{theorem}{Theorem}

\usepackage{xcolor}

\newtheorem{definition}{Definition}

\title{SAIL: Self-Augmented Graph Contrastive Learning}
\author {
    Lu Yu\textsuperscript{\rm 2,1},
    Shichao Pei\textsuperscript{\rm 1},
    Lizhong Ding\textsuperscript{\rm 5},
    Jun Zhou\textsuperscript{\rm 2},
    Longfei Li\textsuperscript{\rm 2},\\
    Chuxu Zhang$^{4*}$,
    Xiangliang Zhang\textsuperscript{\rm 3,1}\footnote{Corresponding author.}
}
\affiliations {
    \textsuperscript{\rm 1} King Abdullah University of Science and Technology, Saudi Arabia\\
    \textsuperscript{\rm 2} Ant Group, Hangzhou, China\\
    \textsuperscript{\rm 3} University of Notre Dame, USA\\
    \textsuperscript{\rm 4} Brandeis University, USA\\
    \textsuperscript{\rm 5} Inception Institute of Artificial Intelligence, UAE\\
    bruceyu.yl@alibaba-inc.com, shichao,pei@kaust.edu.sa, lizhong.ding@inceptioniai.org,\\
    \{jun.zhoujun,longyao.llf\}@antgroup.com,chuxuzhang@brandeis.edu,xzhang33@nd.edu
}

\usepackage{bibentry}
\begin{document}

\maketitle

\begin{abstract}
This paper studies learning node representations with graph neural networks (GNNs) for unsupervised scenario. Specifically, we derive a theoretical analysis and provide an empirical demonstration about the non-steady performance of GNNs over different graph datasets, when the supervision signals are not appropriately defined. The performance of GNNs depends on both the node feature smoothness and the locality of graph structure. To smooth the discrepancy of node proximity measured by graph topology and node feature, we proposed SAIL - a novel \underline{S}elf-\underline{A}ugmented graph contrast\underline{i}ve \underline{L}earning framework, with two complementary self-distilling regularization modules, \emph{i.e.}, intra- and inter-graph knowledge distillation. We demonstrate the competitive performance of SAIL on a variety of graph applications. Even with a single GNN layer, SAIL has consistently competitive or even better performance on various benchmark datasets, comparing with state-of-the-art baselines.
\end{abstract}

\section{Introduction}
Graph neural networks (GNNs) have been a leading effective framework of learning graph representations. The key of GNNs roots at the repeated aggregation over local neighbors to obtain smoothing node representations by filtering out noise existing in the raw node features. With the enormous architectures proposed~\cite{kipf2016semi,hamilton2017inductive,velivckovic2017graph}, learning GNN models to maintain local smoothness usually depends on supervised signals (\emph{e.g.}, node labels or graph labels). However, labeled information is not always available in many scenarios. Along with the raising attention on self-supervised learning~\cite{velivckovic2018deep,hassani2020contrastive}, pre-training GNNs without labels has become an alternative way to learn GNN models. 

There are a group of unsupervised node representation learning models in the spirit of \emph{self-supervised learning}. As one of the most representatives, predicting contextual neighbors (\emph{e.g.}, DeepWalk~\cite{perozzi2014deepwalk} or node2vec~\cite{grover2016node2vec}) enforces the locally connected nodes to have similar representations. Self-supervising signals of this method are designed to extract local structure dependency but discard the contribution of node feature smoothness which has been utilized to improve expressive capability of GNNs~\cite{kipf2016semi,wu2019simplifying}. Another line pays attention to maximizing the \emph{mutual information} (MI) criterion to make agreement on \emph{multi-view graph representations}~\cite{hassani2020contrastive,you2020graph}, in which each view of augmented graph is generated by operations on nodes, attributes, \emph{etc}. However, most of them aim at making agreement on the graph-level representations~\cite{you2020graph,sun2019infograph}, which might be not suitable for node-level tasks. 
 
Instead of creating multiple views through \emph{graph augmentation}~\cite{hassani2020contrastive} methods, there are recent works building upon \emph{self-augmented} views created by the intermedian hidden layers 
% (\emph{e.g.} $\textbf{H}^l, 0 \leq l \leq L$)
of GNN. As a pioneering work, deep graph infomax (DGI)~\cite{velivckovic2018deep} proposes to maximize MI between the summarized graph 
% (\emph{i.e.} $f(\textbf{H}^L)$)
and node embeddings.
% ($h_i \in \textbf{H}^L$).
However, the summarized graph embedding contains the global context that 
might not be shared by all nodes. Inspired by DGI, graphical mutual information (GMI)~\cite{peng2020graph} turns to maximize the edge MI between the created views of two adjacent nodes. As GMI focuses on the edge MI maximization task, it lacks a bird's eye on the learned node representations. The learned GNN might bias towards performing well on edge prediction task, but downgrades on the other tasks like node clustering or classification. Recently some works~\cite{mandal2021meta} attempt to bring the idea of \emph{meta-learning} to train GNNs with meta knowledge which can help to avoid the bias caused by single pretext task. However, the meta-GNN might contain knowledge that cause task discrepancy issue~\cite{tianconsistent,wang2020meta}.

With the knowledge of the previous work,
we just wonder \textbf{\emph{can we advance the expressivity of GNNs with the knowledge extracted by themselves in an unsupervised way?}} In order to answer this question, we theoretically dissect the graph convolution operations (shown in Theorem~\ref{gcn:high:prox}), and find that the smoothness of node representations generated by GNNs is dominated by smoothness of neighborhood embeddings from previous layers and the structural similarity. It suggests that improving the graph representations of shallow layer can indirectly help get better node embeddings of deep layer or the final layer. Based on this observation, we propose SAIL, a  \emph{\underline{S}elf-\underline{A}ugmented graph contrast\underline{i}ve \underline{L}earning} framework, in which we mainly use two different views (\emph{i.e.}, non-linear mappings of input node feature and the final layer of GNN) of transformed node representations. 

More specifically, we propose to iteratively use the smoothed node representations from the output layer of GNNs to improve the node representations of shallow layer or the input layer (\emph{e.g.}, non-linear mapping of input node features). The most recent work~\cite{chen2021selfdistilling} also shares a similar idea for supervised task, while the different thing is that it forces the knowledge flow from low-level to high-level neural representations. Besides attempting to make an agreement on the selected views, we introduced a \emph{self-distilling} module to raise consistency regularization over node representations from both local and global perspectives. The design of self-distilling module is inspired by a recent work \cite{wang2020understanding} on the importance of \emph{alignment} and \emph{uniformity} for a successful contrastive learning method. With a given distribution of positive pair, the alignment calculates the expected similarity of connected nodes (\emph{i.e.} locally closeness), and the \emph{uniformity} measures how well the encoded node embeddings are globally distributed. The intra-distilling module aims at forcing the learnt representations and node features to have consistent uniformity distribution. Inspired by another piece of work~\cite{ishida2020we} indicating out the alleviating the learning bias by injecting noise into objective, we design an inter-distilling framework to align the node representations from a copied teacher model to noisy student model. Through multiple runs of the inter-distilling module, we implicitly mimic the deep smoothing operation with a shallow GNN (\emph{e.g.}, only a single GNN layer), while avoiding noisy information from high-order neighbors to cause the known oversmoothing issue~\cite{chen2019measuring,li2018deeper} since shallow GNN only depends on the local neighbors. The proposed SAIL can learn shallow but powerful GNN. Even with a single GNN layer, it has consistently competitive or even better performance on various benchmark datasets, comparing to state-of-the-art baselines. We summarize the contributions of this work as follow:
\begin{itemize}
\item[-] We present SAIL, to the best of our knowledge, the first generic self-supervised framework designed for advancing the expressivity of GNNs through distilling knowledge of self-augmented views but not depending on any external teacher model.
\item[-] We introduce a universal self-distilling module for unsupervised learning graph neural networks. The presented self-distilling method can bring several advantages including but not limited to: 1) distilling knowledge from a self-created teacher model following the graph topology; 2) we can mimic a deep smoothing operation with a shallow GNN by iteratively distilling knowledge from teacher models to guide a noisy student model with future knowledge.
\item[-] We demonstrate the effectiveness of proposed method with thorough experiments on multiple benchmark datasets and various tasks, yielding consistent improvement comparing with state-of-the-art baselines.
\end{itemize}

\section{Related Work} \label{sec:relatedwork}
\textbf{Graph Neural Networks.}  In recent years, we have witnessed a fast progress of graph neural network in both methodology study and its applications. As one of the pioneering research, spectral graph convolution methods~\cite{defferrard2016convolutional} generalized the convolution operation to non-Euclidean graph data. Kipf \emph{et al.}~\cite{kipf2016semi} reduced the computation complexity to 1-order Chebyshev approximation with an affined assumption.  NT \emph{et al.}~\cite{Hoang:2019:LPF} justified that classical graph convolution network (GCN) and its variants are just low-pass filter. At the same time, both of studies~\cite{li:2019:label,klicpera:2019:diffusion} proposed to replace standard GCN layer with a normalized high-order low-pass filters (\emph{e.g.} personalized PageRank, heats kernel). This conclusion can help to answer why simplified GCN proposed by Wu \emph{et al.}~\cite{wu2019simplifying} has competitive performance with complicated multi-layer GNNs. Besides GCNs, many novel GNNs have been proposed, such as multi-head attention models~\cite{velivckovic2017graph,ma2019disentangled}, recurrent graph neural network~\cite{liu2019geniepath}, RevGNN~\cite{li2021training}, heterogeneous graph neural network~\cite{zhang2019heterogeneous}.

\noindent
\textbf{Self-supervised Learning for GNNs.}  
In addition to the line of work following Deepwalk~\cite{perozzi2014deepwalk} for constructing self-supervising signals, mutual information (MI) maximization~\cite{velivckovic2018deep} over the input and output representations shows up as an alternative solution. In analogy to discriminating that output image representation is generated from the input image patch or noisy image, Veli{\v{c}}kovi{\'c} \emph{et al.}~\cite{velivckovic2018deep} propose deep graph infomax (DGI) criterion to maximize the mutual information between a high-level ``global" graph summary vector and a ``local" patch representation. Inspired by DGI, more and more recent methods like InfoGraph~\cite{sun2019infograph}, GraphCL~\cite{you2020graph}, GCC \cite{qiu2020gcc} and pre-training graph neural networks~\cite{hu2019strategies} are designed for learning graph representations. 
Most of self-supervised learning methods can be distinguished by the way to do data augmentation and the predefined pretext tasks \cite{xie2021self,sun2020multi,you2020does,zhao2021multi,xu2021infogcl}. For example, graph contrastive coding (GCC) borrows the idea from the momentum contrastive learning \cite{he2020momentum}, and aims at learning transferrable graph neural networks taking the node structural feature as the input. Both of GraphCL \cite{you2020graph} and InfoGraph \cite{sun2019infograph} are designed for learning agreement of graph-level representations of augmented graph patches.

\begin{figure*}[tp]
    \centering
    \label{overall:arch}
    \includegraphics[scale=0.6]{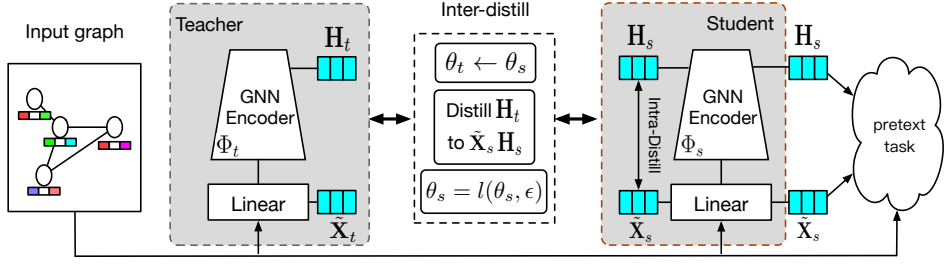}
    % \vspace{-0.1in}
    \caption{Overall architecture of the proposed self-supervised GNN with intra- and inter-distilling modules. The intra-distilling module aims at forcing the learnt representations and node features to have consistent uniformity distribution. The inter-distilling module consists of three operations including: 1) creating a teacher model by copying the target model (\emph{i.e.}, $\theta_t \leftarrow \theta_s$), 2) fading the target model to a student model by injecting noise into the model parameters ($l(\theta_s, \epsilon) = w\theta_s + (1 - w)\epsilon$), 3) supervising the faded target model (\emph{i.e.}, student model) with future knowledge (\emph{i.e.}, $\textbf{H}_t$).}
%    \vspace{-0.1in}
    \label{fig:my_label}
\end{figure*}
\section{Methodology} \label{sec:mtd}
Let $\mathcal{G}=\{\mathcal{V}, \mathcal{E}, \textbf{X}\}$ denote an   attribute graph where $\mathcal{V}$ is the node set  $\{v_i \in \mathcal{V}\}$,  $\mathcal{E}$ is the edge set, and $\textbf{X}\in \mathbb{R}^{N\times F}$ is the node \emph{feature matrix} where each row $\textbf{x}_i$ stands for the \emph{node} feature of $v_i$. We use \textbf{A} represents the node \emph{relation matrix}  where  $a_{ij}=1$ if there existing a link between node $v_i$ and $v_j$, i.e., $e_{ij}\in \mathcal{E}$, otherwise $a_{ij}=0$. We define the degree matrix $\textbf{D} = diag(d_1, d_2,\cdots,d_N)$ where each element equals to the row-sum of adjacency matrix $d_i=\sum_j a_{ij}$. 

Our goal is to learn a graph neural \emph{encoder}  $\Phi(\textbf{X}, \textbf{A}|\theta) = \textbf{H}$, where $\textbf{H} \equiv \{h_1, h_2, \cdots, h_N\}$ is the representation learned for nodes in $\mathcal{V}$. Deep graph encoder $\Phi$ usually has oversmoothing problem. In this work, we instantiate a GNN with single layer to validate the effectiveness of proposed method to learn qualified node representations with shallow neighbors. But the analysis results in the following section can be easily extended to deeper GNNs. The graph neural encoder~\cite{li:2019:label,wu2019simplifying} used in this paper is defined as:
\begin{equation}
\label{eq:gcn}
\begin{aligned}
&\tilde{\textbf{A}}=D^{-\frac {1} {2}} (A+I) D^{-\frac {1} {2}}\\
 & \Phi(\textbf{X}, \tilde{\textbf{A}})= \sigma(\tilde{\textbf{A}}^2 \textbf{X} \textbf{W})\\
\end{aligned}
\end{equation}
where $\textbf{W}\in \mathbb{R}^{F\times F'}$ is learnable parameter and $\sigma(\cdot)$ denotes the activation function.
The vector $h_i\in \mathbb{R}^{F'}$ actually summarizes a subgraph containing the second-order neighbors centered around node $v_i$. We refer to $\textbf{H}$ and $\widetilde{\textbf{X}}=\textbf{X} \textbf{W}$ as the \emph{self-agumented node representations} after transformed raw node featatures $\textbf{X}$. $\widetilde{\textbf{X}}$ denotes the low-level \emph{node feature}, which might contain lots of noisy information. 

\begin{definition}[Second-order Graph Regularization]
The objective of second-order graph regularization is to  minimize the following equation
\begin{equation}
\sum_{e_{ij} \in \mathcal{E}} s_{ij}|| h_i - h_j ||_2^2
\end{equation}
where $s_{ij}$ is the second-order similarity which can be defined as cosine similarity $s_{ij} = \frac {\sum_{c\in \mathcal{N}(i)\cap \mathcal{N}(j)} \alpha_{ic} \cdot \alpha_{jc}} {||\alpha_{i\cdot}||_2 ||\alpha_{j\cdot}||_2}$, and $h_i$ denotes the node representation.
\end{definition}

\begin{theorem}\label{gcn:high:prox}
Suppose that a GNN aggregates node representations as $h_i^l = \sigma(\sum_{j\in \mathcal{N}(i)\cap v_i} \alpha_{ij} h_j^{l-1})$, where $\alpha_{ij}$ stands for the element of a normalized relation matrix. If the first-order gradient of the selected activation function $\sigma(x)$ satisfies $|\sigma'(x)| \leq 1$, then the graph neural operator approximately equals to a second-order proximity graph regularization over the node representations.
\end{theorem}
\begin{proof}
Here we mainly focus on analyzing GNNs whose aggregation operator mainly roots on weighted sum over the neighbors, \emph{i.e.} $h^l = \sigma(\sum_{j\in\mathcal{N}(i)\cap v_i} \alpha_{ij} h_j^{l - 1})$. Typical examples include but not limited to GCN~\cite{kipf2016semi} where $\alpha_{ij}$ can be the element of normalized adjacent matrix $\tilde{\textbf{A}}=D^{-\frac {1} {2}} (A+I) D^{-\frac {1} {2}}$ or $\tilde{\textbf{A}}^2$. The node representation $h_i^l$ can be divided into three parts: the node representations $\alpha_{ii} h_i^{l-1}$, the sum of common neighbor representations $\mathcal{S}_i=\sum_{c\in \mathcal{N}(i)\cap \mathcal{N}(j)} \alpha_{ic} h_c^{l-1}$, the sum of non-common neighbor representations $\mathcal{D}_i = \sum_{q\in \mathcal{N}(i) - \mathcal{N}(i)\cap \mathcal{N}(j)} \alpha_{iq} h_q^{l-1}$. Let $y = \sigma(x)$, and suppose that the selected activation function holds $|\sigma'(x)| \leq 1$. We can have $\frac {(y_1 - y_2)^2} {(x_1 - x_2)^2} = \frac {|y_1 - y_2|^2} {|x_1 - x_2|^2} \leq 1$. Let's reformulate the definition of $h^l$  as $ h^l = \sigma(\hat{h}^l)$ and $\hat{h}^l = \sum_{j\in\mathcal{N}(i)\cap v_i} \alpha_{ij} h_j^{l - 1}$. Then we can have $||h_i^l - h_j^l ||_2 \leq ||\hat{h}_i^l - \hat{h}_j^l ||_2.$ The distance between the representations $h_i^l$ and $h_j^l$ satisfies:
\begin{equation}
\label{eq:distance}
\begin{aligned}
& ||h_i^l - h_j^l ||_2 \leq ||\hat{h}_i^l - \hat{h}_j^l ||_2 \\
&= ||(\alpha_{ii}h_i^{l-1} - \alpha_{jj} h_j^{l-1}) + (\mathcal{S}_i - \mathcal{S}_j) + (\mathcal{D}_i - \mathcal{D}_j)||_2\\
	&  \leq ||(\alpha_{ii}h_i^{l-1} - \alpha_{jj} h_j^{l-1})||_2 + ||(\mathcal{S}_i - \mathcal{S}_j)||_2 +||(\mathcal{D}_i - \mathcal{D}_j)||_2\\
	& \leq \underbrace{||(\alpha_{ii}h_i^{l-1} - \alpha_{jj} h_j^{l-1})||_2}_{local\ feature\ smoothness} + \underbrace{||\mathcal{D}_i||_2 + ||\mathcal{D}_j||_2}_{non-common\ neighbor}\\
	& + \underbrace{||\sum_{c\in \mathcal{N}(i)\cap \mathcal{N}(j)} (\alpha_{ic} - \alpha_{jc}) h_c^{l-1}||_2}_{structure\ proximity} \\
\end{aligned}
\end{equation}
\end{proof}
From Equation~\ref{eq:distance}, we can see that the upper bound of similarity of a pair of nodes is mainly influenced by \emph{local feature smoothness} and \emph{structure proximity}. According to the proof shown above, if a pair of node $(v_i, v_j)$ has smoothed local features and similar structure proximity with many common similar neighbors ($i.e.\; \alpha_{ic} \approx \alpha_{jc}$), the obtained node representation of a GNN will also enforce their node representations to be similar.

\subsection{Learning from Self-augmented View}\label{ssl:gnn}

From the conclusion given in Theorem~\ref{gcn:high:prox}, we can see that the quality of each GNN layer has close relation to previous layer. As the initial layer, the quality of input layer feature $\tilde{\textbf{X}}$ will propagate from the bottom to the top layer of a given GNN model. As a graph neural layer can work as a low-pass filter~\cite{Hoang:2019:LPF}, its output $\textbf{H}$ are actually smoothed node representations after filtering out the noisy information existing in the low-level features. Usually single GNN layer might not perfectly get a clean node representations. By stacking multiple layers, a deep GNN model can repeatedly improved representations from previous layer. However, deep GNN models tend to oversmooth the node representations with unlimited neighborhood mixing. In this work, we attempt to improve the GNN model with shallow neighborhood by shaping the low-level node features with relatively smoothed node representations.

To overcome the above-discussed challenges, we propose to transfer the learnt knowledge in the last GNN layer $\textbf{H}$ to shape $\widetilde{\textbf{X}}$ in both local and global view. Concretely, instead of constructing contrastive learning loss over the node representations $h$ at the same GNN layer, we turn to maximize the neighborhood predicting probability between a node representation $h$ and its input node features $\widetilde{x}$ in its neighbors. Formally, for a sample set $\{v_i, v_j, v_k\}$ where  $e_{ij} \in \mathcal{E}$ but $e_{ik} \notin \mathcal{E}$, the loss $\ell_{jk}^i$ is defined on the pairwise comparison of  $(h_i, \widetilde{x}_j)$ and $(h_i, \widetilde{x}_k))$. Therefore, our self-supervised learning for GNN has the loss function defined below,
\begin{equation}
\label{ssl:obj:1}
\mathcal{L}_{ssl} = \sum_{e_{ij}\in \mathcal{E}} \sum_{e_{ik}\notin \mathcal{E}} - \ell(\psi(h_i, \widetilde{x}_j), \psi(h_i, \widetilde{x}_k)) + \lambda \mathcal{R}(\mathcal{G}),
\end{equation}
where $\ell(\cdot)$ can be an arbitrary contrastive loss function, $\psi$ is a scoring function, and $\mathcal{R}$ is the regularization function with weight $\lambda$ for implementing graph structural constraints (to be introduced in next section). There are lots of candidates for contrastive loss $\ell()$. In this work we use logistic pairwise loss $\ln \sigma(\psi(h_i, \widetilde{x}_j) - \psi(h_i, \widetilde{x}_k))$, wehre $\sigma(x) = \frac {1} {1 + exp(-x)}$. 

\subsection{Self-distilling Graph Knowledge Regularization}

The objective function defined in Eq. (\ref{ssl:obj:1}) models the interactions between output node representations $h$ and input node features $\widetilde{x}$, which can be regarded as an \emph{intra-model knowledge distillation} process from smoothed node embeddings to denoise low-level node features. However, the raised contrastive samples over edge connectivity might fail to represent a whole picture of node representation distribution, and cause a bias to learning node representations favoring to predict edges. We present a self-distilling method shown in Figure 1 consists of intra- and inter-distilling modules.

\begin{algorithm}[ht]
 \caption{$\mathsf{\textsc{SAIL}}$}
 \label{pair:rank}
 \KwData \ graph $\mathcal{G}=\{\mathcal{V}, \mathcal{E}, \textbf{X}\}$, hyperparameters$=\{\alpha,\lambda\}$\\
 \KwResult \ learned GNN $\Phi$\\
 %randomly initialize model parameters $\theta$\\
 initialize $\Phi_0$ by optimizing Eq. (\ref{eq:ssl:obj}) without $R_{cross}$\\
 \For{$m \leftarrow 1$ to $n$ }{
 	\If {$m \% \tau == 0$}
		{$\theta_t \leftarrow \theta_s$\;
		$\theta_s \leftarrow w\theta_s + (1 - w)\epsilon$\; 
		}
	
 	$Optimize\ \mathcal{L}_{ssl}(\Phi_t, \Phi_s, \textbf{X}, \textbf{A} )$;\\
 }
return $\Phi_s$\;
\end{algorithm}

\textbf{Intra-distilling module:} To supplement the loss defined on the individual pairwise  samples, we introduce a regularization term to ensure the distribution consistency on the relations between the learned node representations $\textbf{H}$ and the node features $\widetilde{\textbf{X}}$ over a set of randomly sampled nodes.
Let $LS=\{LS_1, LS_2, \cdots, LS_N\}$ denote the randomly sampled pseudo relation graph, where $LS_i \subset \mathcal{V}$ and $|LS_i|=d$ is the number of sampled pseudo local neighbors for center node $v_i$. The estimated proximity for each node in $i$-th local structure $LS_i$ is computed by
\begin{equation}
\begin{aligned}
S_{ij}^t &= \frac {exp(\psi(h_i, \widetilde{x}_j))} {\sum_{v_j \in LS_i} exp(\psi(h_i, \widetilde{x}_j))}\\
S_{ij}^s &= \frac {exp(\psi(\widetilde{x}_i, \widetilde{x}_j))} {\sum_{v_j \in LS_i} exp(\psi(\widetilde{x}_i, \widetilde{x}_j))}
\end{aligned}
\end{equation}
where $S_{ij}^t$ and $S_{ij}^s$ denote the similarity estimated from different node representations between node $v_i$ and $v_j$. The $S_{ij}^t$ will act as the teacher signal to guide the node features $\widetilde{\textbf{X}} = \{\widetilde{x}_1, \widetilde{x}_2, \cdots, \widetilde{x}_N\}$ to agree on the relation distribution over a random sampled graph. For the node $v_i$, the relation distribution similarity can be measured as $\mathcal{S}_i = CrossEntropy(S_{[i, \cdot]}^t, S_{[i, \cdot]}^s).$
Then we can compute the relation similarity distribution over all the nodes as 
\begin{equation}
\mathcal{R}_{intra} = \sum_{i = 1}^N \mathcal{S}_i,
\end{equation}
where $\mathcal{R}_{intra}$ acts as a regularization term generated from the intra-model knowledge, and to push the learned node representations $\textbf{H}$ and node features $\widetilde{\textbf{X}}$ being consistent at a subgraph-level. 

\textbf{Inter-distilling Module:} The second regularization is to introduce the \emph{inter-distilling} module for addressing the over-smoothing issue. The inter-distilling module can guide the target GNN model by transferring the learned self-supervised knowledge. Through multiple implementations of the inter-distilling  module, we implicitly mimic the deep smoothing operation with a shallow GNN (\emph{e.g.} a single GNN layer), while avoiding to bring noisy information from high-order neighbors. The overall inter-distilling framework is shown in Figure 1. We create a teacher model $\Phi_t$ by copying the target GNN model, then inject noise into the target model that will degrade into a student model $\Phi_s$ after a fix number of iterations. Working with a self-created teacher and student model $\{\Phi_t, \Phi_s\}$ with the same architectures shown in Eq. (\ref{eq:gcn}), student model $\Phi_s(\textbf{X}, \textbf{A}) = \{\textbf{H}_s, \widetilde{\textbf{X}}_s\}$ distills knowledge from the teacher model $\Phi_t$. Since no label is available, we propose to implement \emph{representation distillation}~\cite{tian2019contrastive} with the constraint of graph structure. The knowledge distillation module consists of two parts, defined as
\begin{equation}
\mathcal{R}_{inter} = KD(\textbf{H}_t, \widetilde{\textbf{X}}_s | \mathcal{G}) + KD(\textbf{H}_t, \textbf{H}_s | \mathcal{G})
\end{equation}
where $\textbf{H}_t$ is the node representations from teacher model $\Phi_t$, and $\widetilde{\textbf{X}}_s$ = $\textbf{X}\textbf{W}$. To define    module $KD(\cdot)$, we should meet several requirements: 1) this function should be easy to compute and friendly to back-propagation strategy; and 2) it should stick to  the graph structure constraint. We resort to the conditional random field (CRF)~\cite{lafferty2001conditional} to capture the pairwise relationship between different nodes. 
For a general knowledge distillation module $KD(Y,Z|G)$, the dependency of $Z$ on $Y$ can be given following the CRF model:
\begin{equation}
\label{crf:p}
P(Z | Y) = \frac {1} {C(Y)} exp(-E(Z | Y)),
\end{equation}
where $C(\cdot)$ is the normalization factor and $E(\cdot)$ stands for the energy function, defined as follows:
\begin{equation}
\label{reg:crf}
\begin{aligned}
E(Z_i | Y_i) &= \psi_u(Z_i, Y_i) + \psi_p(Z_i, Z_j, Y_i, Y_j)\\
&=(1-\alpha) || Z_i - Y_i ||_2^2 + \alpha \sum_{j\in \mathcal{N}(i)} \beta_{ij} ||Z_i - Z_j||_2^2,
\end{aligned}
\end{equation}
where $\psi_u$ and $\psi_p$ are the unary and pairwise energy function, respectively. The parameter $\alpha \in [0,1]$ is to control the importance of two energy  functions.
When $Z$ is the node feature of student model and $Y$ is the node representation from teacher model $\Phi_{m-1}$,  the energy function defined in Equation~(\ref{reg:crf})    enforces the node $v_i$ representation from student model to be close to that in the teacher model and its neighbor nodes. After obtaining the energy function, we can resolve the CRF objective with the mean-field approximation method by employing a simple distribution $Q(Z)$ to approximate the distribution $P(Z | Y)$. Specifically distribution $Q(Z)$ can be initialized as the product marginal distributions as $Q(Z)=\Pi_{i=1}^N Q_i(Z_i)$. Through minimizing the KL divergence between these two distributions as follows:
\begin{equation}
\argmin KL(Q(Z) || P(Z | Y)).
\end{equation}
Then we can get the optimal $Q_i^*(Z_i)$ as follows:
\begin{equation}
\ln Q_i^*(Z_i) = \mathbb{E}_{j\neq i}[\ln P(Z_j | Y_j)] + const.
\end{equation}
According to Equation~(\ref{crf:p}) and (\ref{reg:crf}), we can get 
\begin{equation}
\label{crf:opt:q}
Q_i^*(Z_i)\sim exp((1-\alpha) || Z_i - Y_i ||_2^2 + \alpha \sum_{j\in \mathcal{N}(i)} \beta_{ij} ||Z_i - Z_j||_2^2),
\end{equation}
which shows that $Q_i^*(Z_i)$ is a Gaussian function. By computing its expectation, we have the optimal solution for $Z_i$ as follows:
\begin{equation}
\label{crf:com}
Z_i^{*} = \frac {(1-\alpha) Y_i + \alpha \sum_{j \in \mathcal{N}(i)} \beta_{ij} Z_j} {(1-\alpha) + \alpha \sum_{j\in \mathcal{N}(i)} \beta_{ij}}
\end{equation}
Then we can get the cross-model knowledge distillation rule by enforcing the node representations from student model to have minimized metric distance to $Z_i^*$. After replacing the random variables $Y_i$ as the node representation $h_{i}^t$ of teacher model $\Phi_t$, then we can get the final distillation regularization as follows:
\begin{equation}
\begin{aligned}
&KD(\textbf{H}_t, \widetilde{\textbf{X}}_s | \mathcal{G}) = ||\widetilde{x}_i^s - \frac {(1-\alpha) h_i^t + \alpha \sum_{j \in \mathcal{N}(i)} \beta_{ij} \widetilde{x}_j^s} {(1-\alpha) + \alpha \sum_{j\in \mathcal{N}(i)} \beta_{ij}} ||_2^2
\end{aligned}
\end{equation}
\begin{equation}
\begin{aligned}
&KD(\textbf{H}_t, \textbf{H}_s | \mathcal{G}) = ||h_i^s - \frac {(1-\alpha) h_i^t + \alpha \sum_{j \in \mathcal{N}(i)} \beta_{ij} h_j^s} {(1-\alpha) + \alpha \sum_{j\in \mathcal{N}(i)} \beta_{ij}} ||_2^2
\end{aligned}
\end{equation}
where $\widetilde{x}_i^s$ denotes the  feature of node $v_i$ from the student model $\Phi_s$, and $h_i^s$ denotes the output node representation for node $v_i$. In terms of $\beta_{ij}$ in Eq. (\ref{crf:com}), we have many choices  such as attentive weight, or mean pooling \emph{etc.} In this work, we simply initialize it with mean-pooling operation over the node representations. 
\begin{figure}[tp]
\centering
 \includegraphics[trim={0 0 0 45}, scale=0.4]{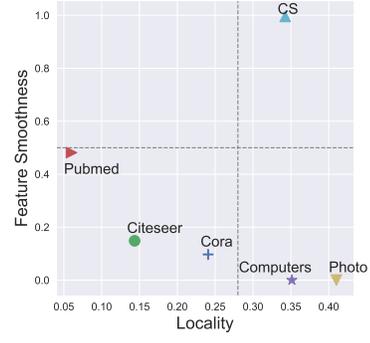}\\
%\vspace{-0.2cm}
\caption{Locality and node feature smoothness in the six graphs used in experimental evaluation. The larger feature smoothness value is, the more smoothing feature we have.}
\label{fig:loc:cc}
%\vspace{-0.35cm}
\end{figure}

The overall self-supervised learning objective in this work can be extended as follows by taking the two regularization terms:
\begin{equation}
\centering
\label{eq:ssl:obj}
\begin{aligned}
\mathcal{L}_{ssl} = \sum_{e_{ij}\in \mathcal{E}} \sum_{e_{ik}\notin \mathcal{E}} & - \ell(\psi(h_i^s, \widetilde{x}_j^s), \psi(h_i^s, \widetilde{x}_k^s)) \\
& + \lambda (\mathcal{R}_{intra} + \mathcal{R}_{inter}).
\end{aligned}
\end{equation}
where the initial teacher model $\Phi_t$ can be initialized through optimizing the proposed self-supervised learning without the cross-model distillation regularization.

\begin{table*}[htp]
\caption{Accuracy (with standard deviation) of node classification (in $\%$). The best results are highlighted in bold. Results without standard deviation are copied from original works.}
% \vspace{-0.1in}
\label{result:node:clf}
\centering
\footnotesize
\begin{adjustbox}{max width=16.5cm}
\begin{tabular}{c |c| c c c c c c}
\toprule
	\textbf{Input} &\textbf{Methods} & \textbf{Cora} & \textbf{Citeseer} & \textbf{PubMed} & \textbf{Computers} & \textbf{Photo} & \textbf{CS}\\
\midrule
	\multirow{8}{*}{\textbf{X},\textbf{A},\textbf{Y}} &ChebNet~\cite{defferrard2016convolutional} & 81.2 & 69.8 & 74.4 & 70.5$\pm$0.5 & 76.9$\pm$0.3 & 92.3$\pm$0.1\\
	&MLP~\cite{kipf2016semi} & 55.1 & 46.5 & 71.4 & 55.3$\pm$0.2 & 71.1$\pm$0.1 & 85.5$\pm$0.1\\
	&GCN~\cite{kipf2016semi} & 81.5 & 70.3 & 79.0 & 76.11$\pm$0.1 & 89.0$\pm$0.3 & 92.4$\pm$0.1\\
	&SGC~\cite{wu2019simplifying} & 81.0 $\pm$ 0.0 & 71.9$\pm$0.1 & 78.9$\pm$0.0 & 55.7$\pm$0.3 & 69.7$\pm$0.5 & 92.3$\pm$0.1\\
	&GAT~\cite{velivckovic2017graph} & 83.0$\pm$0.7 & 72.5$\pm$0.7 & 79.0$\pm$0.3 & 71.4$\pm$0.6 & 89.4$\pm$0.5 & 92.2$\pm$0.2\\
	&DisenGCN~\cite{ma2019disentangled} & 83.7* & 73.4* & 80.5 & 52.7$\pm$0.3 & 87.7$\pm$0.6 & 93.3$\pm$0.1*\\
	&GMNN~\cite{qu2019gmnn} & 83.7* & 73.1 & 81.8* & 74.1$\pm$0.4 & 89.5$\pm$0.6 & 92.7$\pm$0.2\\
	&GraphSAGE~\cite{hamilton2017inductive} & 77.2$\pm$0.3 & 67.8$\pm$0.3 & 77.5 $\pm$0.5 & 78.9$\pm$0.4 & 91.3$\pm$0.5  &93.2 $\pm$0.3 \\
\midrule
	\multirow{5}{*}{\textbf{X},\textbf{A}}
	&CAN~\cite{meng:2019:can}& 75.2$\pm$0.5 & 64.5$\pm$0.2 & 64.8$\pm$0.3 & 78.2$\pm$0.5 & 88.3$\pm$0.6 & 91.2$\pm$0.2 \\
	&DGI~\cite{velivckovic2018deep} & 82.3$\pm$0.6 & 71.8$\pm$0.7 & 76.8$\pm$0.6 & 68.2$\pm$0.6 & 78.2$\pm$0.3 & 92.4$\pm$0.3\\
	&GMI~\cite{peng2020graph} & 82.8$\pm$0.7 & 73.0$\pm$0.3 & 80.1$\pm$0.2 & 62.7$\pm$0.5 & 86.2$\pm$0.3 & 92.3$\pm$0.1 \\
	&MVGRL~\cite{hassani2020contrastive}\ & 83.5$\pm$0.4 & 72.1$\pm$0.6 & 79.8$\pm$0.7 & 88.4$\pm$0.3* & \textbf{92.8}$\pm$0.2 & 93.1$\pm$0.3 \\
	& SAIL & \textbf{84.6}$\pm$0.3 & \textbf{74.2}$\pm$0.4 & \textbf{83.8}$\pm$0.1 & \textbf{89.4}$\pm$0.1 & 92.5$\pm$0.1* & \textbf{93.3}$\pm$0.05\\
\bottomrule
\end{tabular}
\end{adjustbox}
% \vspace{-0.2cm}
\end{table*}
\section{Experimental Evaluation}

We compare the proposed method with various state-of-the-art methods on six datasets including three citation networks (Cora, Citeseer, Pubmed)~\cite{kipf2016semi}, product co-purchase networks (Computers, Photo), and one co-author network subjected to a subfield Computer Science (CS).  The product and co-author graphs are benchmark datasets from pytorch\_geometric~\cite{fey2019fast}. Due to the limited space, more details about the experimental setting can be found in the appendix.

\begin{table}[!tp]
\caption{Node clustering performance measured by normalized mutual information (NMI) in \%.}
%\vspace{-0.1in}
\label{result:clf:cluster}
\centering
\footnotesize
\begin{adjustbox}{max width=8.5cm}
\begin{tabular}{ c | c c c c c c}
\toprule
	\textbf{Methods} & \textbf{Cora} & \textbf{Citeseer} & \textbf{PubMed} & \textbf{Computers} & \textbf{Photo} & \textbf{CS}\\
\midrule
	CAN & 51.7 & 35.4 & 26.7 & 42.6 & 53.4* & 71.3\\
	GMI & 55.9 & 34.7 & 23.3 & 34.5 & 47.2 & 73.9\\
	DGI & 50.5 & 40.1* & 28.8 & 36.8 & 42.1 & 74.7* \\
	MVGRL & 56.1* & 37.6 & \textbf{34.7} & 46.2* & 12.15 & 66.5 \\
	SAIL & \textbf{58.1} & \textbf{44.6} & 33.3* & \textbf{49.1} & \textbf{66.5} & \textbf{76.4}\\
\bottomrule
\end{tabular}
\end{adjustbox}
%\vspace{-0.1in}
\end{table}

\begin{table}[tp]
\caption{AUC (in \%) of link prediction.}
%\vspace{-0.1in}
\label{link:clf}
\centering
\footnotesize
\begin{adjustbox}{max width=8.5cm}
\begin{tabular}{c | c c c c c c}
\toprule
	\textbf{Methods} & \textbf{Cora} & \textbf{Citeseer} & \textbf{PubMed} & \textbf{Computers} & \textbf{Photo} & \textbf{CS}\\
\midrule
	GMNN & 87.5 & 86.9& 88.8 &82.1 & 86.7 & 91.7 \\
	GAT & 90.5 & 89.1 & 80.5 & 84.5 & 88.4 & 92.2\\
	GCN& 82.6& 83.2& 88.5 &82.1& 86.7& 89.7\\
	DisenGCN& 93.3& 92.0 & 91.1& 78.9 & 77.6 & 94.5\\
\midrule
	DGI & 69.2 & 69.0 & 85.2 & 75.1 & 74.2 & 79.7\\
	MVGRL & 89.5 & 94.4 & 96.1* & 74.6 & 73.1 & 83.1\\
	CAN & 94.8 & 94.8 & 91.9& 94.9* & \textbf{95.0} & 97.1* \\
	GMI & 95.1* & 96.0* & 96.0 & 85.5 & 91.9 & 95.5\\
	SAIL & \textbf{97.3} & \textbf{98.4} & \textbf{98.5} & \textbf{94.9} & 94.6* & \textbf{97.4}\\
\bottomrule
\end{tabular}
%\vspace{-0.1in}
\end{adjustbox}
\end{table}
\subsection{Overall Performance} 

\subsubsection{Node Classification.} The GNN methods compared here include convolution or attentive neural networks. Table~\ref{result:node:clf} summarizes the overall performance. The performance of simple GCN learned by the proposed method SAIL consistently outperforms the other baselines learned by supervised and unsupervised objectives. It's noted that DisenGCN iteratively applies attentive routing operation to dynamically reshape node relationships in each layer. By default, it has 5-layers and iteratively applies 6 times, which typically is a deep GNN. From the empirical results, we can see that SAIL can empower the single-layer GNN through iterative inter-model knowledge distillation. 

\subsubsection{Node Clustering.} In node clustering task, we aim at evaluating the quality of the node representations learned by unsupervised methods. The performance is validated by measuring the \emph{normalized mutual information} (NMI). The node representations are learned with the same experimental setting for the node classification task. From the results shown in Table~\ref{result:clf:cluster}, we can see that the proposed method are superior to the baselines in most of cases.

\subsubsection{Link Prediction.} In addition, we attempt answer the question about whether the learned node representations can keep the node proximity. The performance of each method is measured with AUC value of link prediction. All of the methods in this experiment have the same model configuration as the node classification task. From the results shown in Table~\ref{link:clf}, we can see that SAIL still outperforms most of the baselines learned with supervised and unsupervised objectives. From the results in Table~\ref{result:node:clf}, we see that classification of nodes in graph \emph{CS} is indeed an easy task, and most of the GNNs models have similar good performance. However, for link prediction shown in Table~\ref{link:clf}, unsupervised models (CAN and our model) learned better $h$ than those with supervision information, obviously because the supervision information is for node classification, not for link prediction. 
\subsection{Exploring Locality and Feature Smoothness}
Based on our theoretical understanding in Theorem 1, the representation smoothness is mainly influenced by the local structural proximity and feature closeness. Follow the ideas of recent works~\cite{Hou2020Measuring,chen2019measuring}, we conduct empirical study on the node representation smoothness before and after being encoded by the GNNs.

\subsubsection{Before encoding.} With a given graph and node features, we calculate the inverse of average pairwise distance among the raw feature~\cite{Hou2020Measuring}, and clustering coefficient~\cite{watts1998collective} to study the feature smoothness and structural locality, respectively. Combining with the node classification results, we empirically found that most of the neural encoders (e.g., GNNs, even simple MLPs) performed well on node classification in graphs like CS, which has a strong locality and large feature smoothness shown in Figure~\ref{fig:loc:cc}. Interestingly, for graphs with a strong locality but a low node feature smoothness (e.g.,  ``Computers", ``Photo"), unsupervised methods can leverage the graph structure to achieve better performance than supervised methods. 

 \begin{table*}[!tp]
\caption{Accuracy (in $\%$) of node classification after randomly 20\% neighbors removal. The last column shows the average classification accuracy downgrades comparing with the results in Table~\ref{result:node:clf}.}
%\vspace{-0.1in}
\label{node:link:clf}
\centering
\footnotesize
\begin{adjustbox}{max width=11cm}
\begin{tabular}{c |c| c c c c c c c}
\toprule
	\textbf{Input} &\textbf{Methods} & \textbf{Cora} & \textbf{Citeseer} & \textbf{PubMed} & \textbf{Computers} & \textbf{Photo} & \textbf{CS}& \textbf{Avg $\downarrow$}\\
\midrule
	\multirow{4}{*}{\textbf{X},\textbf{A},\textbf{Y}}
	&GMNN& 77.2  & 68.8  & 79.8  & 70.8 & 87.5  & 91.6 & 4.1\%\\  
	&GAT & 77.7 & 65.6  & 76.7  & 69.3 & 89.4 & 90.4  & 4.0\%\\
	&GCN & 76.0  & 67.2  & 77.7  & 67.5  & 88.7& 89.9  & 4.5\%\\
	&DisenGCN & 77.6  & 68.2  & 78.3  & 37.5  & 48.8 & 92.4  & 15.1\%\\
\midrule
\multirow{5}{*}{\textbf{X},\textbf{A}}
	&CAN & 73.2  & 64.0 & 63.5  & 77.5  & 88.1 & 91.1  & 1.1\%\\
	&DGI & 72.3 & 70.1  & 71.5  & 67.6 & 77.4   & 91.7  & 4.0\%\\
	&GMI & 77.4 & 68.3  & 76.9 & 54.9  & 82.5  & 89.2 &6.1\%\\
	&MVGRL & 69.5  & 62.5  & 76.5 & 87.2  & \textbf{92.5}  & 91.8 &6.2\%\\
	&SAIL & \textbf{81.0} & \textbf{71.3}  & \textbf{81.2} & \textbf{88.5}  & 92.4 & \textbf{92.7}  & 2.1\%\\
\bottomrule
\end{tabular}
\end{adjustbox}
 \vspace{-0.35cm}
\end{table*}
\begin{table*}[htp]
\caption{Empirical analysis to shown the quality of learned node representations measured by mean average distance (MAD).}
\vspace{-0.1in}
\label{mad:eval}
\centering
\footnotesize
\begin{adjustbox}{max width=13.5cm}
\begin{tabular}{c | c | c c c c c c c c c}
\toprule
	\textbf{Data} & \textbf{Metrics} & \textbf{GCN} & \textbf{GAT} & \textbf{DisenGCN} & \textbf{GMNN} & \textbf{CAN} & \textbf{DGI} & \textbf{GMI} & \textbf{MVGRL} & \textbf{SAIL}\\
\midrule
	\multirow{3}{*}{\textbf{Cora}}
	&$MAD_{nei}$ & 0.075 & 0.029& 0.215 & 0.088 & 0.059 & 0.312 & 0.069 & 0.240 & 0.013 \\
	&$MAD_{gap}$ & 0.308 & 0.083 & 0.471 & 0.390 & 0.900 & 0.557 & 0.966 & 0.661 & 0.322\\
	&$MAD_{ratio}$& 4.13& 2.86& 2.19 & 4.43 & 15.2& 1.77 & 13.83 & 2.75 & 25.3\\
\midrule
\multirow{3}{*}{\textbf{Citeseer}}
	&$MAD_{nei}$ & 0.049 & 0.014& 0.194 & 0.059 & 0.057 & 0.289 & 0.122 & 0.174 & 0.007 \\
	&$MAD_{gap}$ & 0.308 & 0.083 & 0.471 & 0.390 & 0.921 & 0.497 & 0.879 & 0.491 & 0.429\\
	&$MAD_{ratio}$& 4.13& 2.86& 2.19 & 4.43 & 15.1& 1.72 & 6.2 & 2.82 & 61.2\\
\midrule
\multirow{3}{*}{\textbf{Pubmed}}
	&$MAD_{nei}$ & 0.043 & 0.024& 0.054 & 0.068 & 0.037 & 0.112 & 0.086 & 0.180 & 0.024 \\
	&$MAD_{gap}$ & 0.155 & 0.083 & 0.224 & 0.438 & 0.757 & 0.145 & 0.833 & 0.541 & 0.294\\
	&$MAD_{ratio}$& 6.02& 6.43& 4.17 & 6.39 & 20.6& 1.29 & 8.63 & 3.00 & 12.1\\
\bottomrule
\end{tabular}
\end{adjustbox}
\vspace{-0.3cm}
\end{table*}
\subsubsection{After encoding.}  Chen \emph{et a.}~\cite{chen2019measuring} propose to use mean average distance (MAD) between the target node and its neighbors for measuring the smoothness of node representations, and the MAD gap between the neighbor and remote nodes to measure the over- smoothness. Let's denote $MAD_{gap} = MAD_{rmt} - MAD_{nei}$, where $MAD_{nei}$ represents the MAD between the target node and its neighbors, and $MAD_{rmt}$ is defined to measure the MAD of remote nodes. If we get a small MAD but relatively large MAD gap, then we can say that the learnt node representations are not over-smoothed. We define a variant metric $MAD_{ratio} = \frac {MAD_{gap}} {MAD_nei}$ to measure the \emph{information-to-noise} ratio brought by the relative changes of MAD of remote nodes over neighbors. We use the node representations with the same setting as node classification task. The results shown in Table \ref{mad:eval} demonstrate that the proposed method can achieve the best smoothing performance (\emph{i.e.} smallest $MAD_{nei}$). In terms of over-smoothing issue, SAIL has the relative large $MAD_{gap}$ comparing with the scale of the achieved MAD of each method, which can be measured by $MAD_{ratio}$. The empirical results reflect that the proposed method can help to tell the difference from local neighbors and remote nodes.

\subsubsection{Robustness Against Incomplete Graphs.} With the same experimental configuration as link prediction task, we also validate the performance of learned node embeddings from incomplete graph for node classification task. According to the results in Table~\ref{node:link:clf}, SAIL still outperforms baseline methods in most cases, demonstrating the robustness of SAIL performance. It has a larger average downgrade in terms of node classification accuracy than CAN, but still has better classification accuracy. 
 
\subsection{Ablation Study} We conduct node classification experiments to validate the contribution of each component of the proposed SAIL, where $\mathsf{EMI}$ denotes the edge MI loss $l^i_{jk}$, $\mathsf{Intra}$ stands for intra-distill module $\mathcal{R}_{intra}$ and $\mathsf{Inter}$ represents th $\mathcal{R}_{inter}$  in Eq.~\ref{eq:ssl:obj}. From the results shown in Figure 3, we can see that both intra- and inter-distilling module can jointly improve the qualification of learnt node representations. Combining with the edge MI maximization task, we can see a significant improvement on the node classification accuracy. Due to the limited space, we present the ablation studies on link prediction and node clustering tasks in the appendix.
\begin{figure}[htp]
\vspace{-0.3cm}
\centering
\label{fig:abl:std}
\includegraphics[scale=0.4]{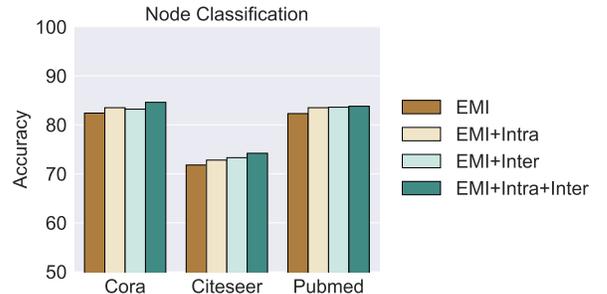}
\vspace{-0.3cm}
\caption{Ablation study of the distillation component influence in node classification accuracy (in $\%$).}
\vspace{-0.5cm}
\end{figure}

\section{Conclusions}
In this work, we propose a self-supervised learning method (SAIL) regularized by graph structure to learn unsupervised node representations for various downstream tasks. We conduct thorough experiments on node classification, node clustering, and link prediction tasks to evaluate the learned node representations. Experimental results demonstrate that SAIL helps to learn competitive shallow GNN which outperforms the state-of-the-art GNNs learned with supervised or unsupervised objectives. This initial study might shed light upon a promising way to implement self-distillation for graph neural networks. In the future, we plan to study how to improve the robustness of the proposed method against adversarial attacks and learn transferable graph neural network for downstream tasks like few-shot classification. 

%For the potential negative societal impact, the proposed method can be used to learn node representations for real networks like social network, knowledge graphs for identifying potential node relationships or types, which might result in the privacy violation.

\clearpage

%\clearpage
%\bibliographystyle{aaai}
\bibliography{gcn-ref-short}

\appendix

\section*{A. Experimental Setting}
\subsection{Baselines}
The baselines in this work include both state-of-the-art supervised and unsupervised learning GNNs. The \textbf{supervised} methods are Chebyshev filters (ChebNet)~\cite{defferrard2016convolutional}, MLP~\cite{kipf2016semi}, Graph Convolution Network (GCN)~\cite{kipf2016semi}, Simplified GCN (SGC)~\cite{wu2019simplifying}, Disentangled Graph Convolution Network (DisenGCN)~\cite{ma2019disentangled}, Graph Attention Network (GAT)~\cite{velivckovic2017graph}, Graph Markov Neural Network (GMNN)~\cite{qu2019gmnn} and GraphSAGE~\cite{hamilton2017inductive}. The \textbf{unsupervised} ones include Deep Graph Infomax (DGI)~\cite{velivckovic2018deep}, CAN~\cite{meng:2019:can}, GMI~\cite{peng2020graph}, and Multi-view Graph Representation Learning (MVGRL)~\cite{hassani2020contrastive}. It's noted that the performance of MVGRL shown in this work is different the published results. In terms of the inconsistent results, we find that many researchers also have the same problem. Please refer to issues at https://github.com/kavehhassani/mvgrl/issues/. From a closed issue about the source of Cora, we can see that the split of Cora used in MVGRL is from DGL, which is different from the Cora used in official implementation of many classical GNN baselines like GCN, GAT. For a fair comparison, we use the official code to implement MVGRL with the benchmark data from https://github.com/tkipf/gcn (https://github.com/tkipf/gcn).

\subsection{Dataset}
The data statistics can be found in Table \ref{stat:data}. We implement the baselines with the official code provided by the authors. If the results are not previously reported, we implement them with suggested hyperparameters according to the original paper. The results shown in this work are average value over 10 runs for each task. For the node classification, the experimental configuration for the three citation data is the same as~\cite{kipf2016semi}. To make a fair comparison on the other co-purchase and co-author networks, we randomly select 20\% nodes for supervised training GNNs, 10\% node as validation set, and the rest of nodes are left for testing. For the unsupervised methods, they can access to the complete graph without using node labels.

For the link prediction task, we randomly select 20\% edges for each node as testing, and the rest as the training set. All baselines involved in this experiment only work on the graph built from the training set. For the supervised learning models like GCN, GAT, DisenGCN, GMNN, we still use the same configure as the node classification task to select labelled nodes to train the model, which means the nodes in the selected testing edges can still be labelled nodes for the supervised training strategy. While the unsupervised methods can only access to the edges in the training set.

\begin{table}[!tp]
\vspace{-0.5cm}
\caption{Summary of data statistics.}
\label{stat:data}
\centering
\footnotesize
\begin{tabular}{c c c c c c c}
\toprule
	\textbf{Data} & \textbf{\#Nodes} & \textbf{\#Edges} & \textbf{\#Features} & \textbf{Labels} \\
\midrule
	 Cora & 2,708 & 5,429 & 1,433 & 7\\
	 Citeseer & 3,327 & 4,732 & 3,703 & 6\\
    Pubmed & 19,717 & 44,338 & 500 & 3 \\
	Computers & 13,752 & 287,209 & 767 & 10\\
	Photo & 7,650 & 143,663 & 745 & 8\\
	CS & 18,333 & 81,894 & 6,805 & 15\\
\bottomrule
\end{tabular}
\vspace{-0.3cm}
\end{table}

\section*{B. Model Configuration}
The proposed method has hyper-paramters including regularization weight $\lambda$, $\alpha$ for balancing the contribution of teacher model and local neighbors. The hyper-parameter $\tau$ controls the number of steps to update the teacher model, while $w$ controls the degradation degree of target model. In this work, we set $\tau$ to 30. Since we use the mean pooling operation, then the $\beta$ equals to $\frac {1} {\mathcal{N}(i)}$ by default.  For each hyper-parameter, we use grid search method to valid the best configuration from pre-defined search space, specifically, $\alpha \in \{0, 0.1, 0.5, 1.0\}$, $\lambda \in \{0.0, 0.1, 0.5, 1.0\}$, $w \in [0.0, 1.0]$, and embedding dimension $F'=512$.

\section*{C. Feature Smoothness}
In this work, we use two kinds of methods to measure the feature smoothness before and after being encoded by GNNs. We use raw feature $\textbf{X} = \{x_i | v_i \in \mathcal{V}\}$ to conduct an empirical analysis to show the potential impact of raw feature smoothness on the difficulty to learn qualified node representations. The specific definition can found as follows.
\begin{definition}[Feature Smoothness~\cite{Hou2020Measuring}] We define the feature smoothness $\lambda_f$ over the normalized space $\mathcal{X}=[0,1]^F$ as follows:
\begin{equation*}
\lambda_f = \frac {|| \sum_{i=1, v_i \in \mathcal{V}}^N \sum_{j \in \mathcal{N}(i)} (x_i - x_j)^2 ||_1} {|\mathcal{E}|\cdot F},
\end{equation*}
where $||\cdot||_1$ denotes the \emph{Manhattan norm}.
\end{definition}
Noted that we use normalized $\frac {1} {\lambda_f}$ for better showing the influence of feature smoothness and locality on the model performance. This metric is used in the ``before encoding" section, which is at the page 6 of submitted manuscript.

The other method is proposed by Chen \emph{et al.}~\cite{chen2019measuring}. They study the over-smoothing issue of encoded node embeddings through measuring the mean average cosine distance (MAD) of node pairs. When attempting to follow the same setting to calculate theMADs, we observe that there exists graphs with very few number of nodes having remote nodes. Therefore, we relax the condition to separate the neighbor and remote nodes. In this work, we modify the nodes that can be reached in larger than 3 steps as the remote nodes. Besides the way to define nearby and remote nodes, we also note that Chen \emph{et al.} use a different way to split the experimental data is very different from classical GCN~\cite{kipf2016semi}. In this work, the obtained MAD values are relative smaller than the results shown in~\cite{chen2019measuring}. We guess the reason lies at the different experimental setting.

\section*{D. Ablation Studies}
We conduct ablation studies on the node clustering and link prediction. The results shown in Table \ref{stat:data:link} and \ref{stat:data:clus} have the similar trend as the results for node classification. Both of inter- and intra- distillation show positive impact on advancing the performance of the base model.
\begin{table}[!tp]
\vspace{-0.5cm}
\caption{Ablation study of the distillation component on link prediction task, measured by AUC in \%.}
\label{stat:data:link}
\centering
\footnotesize
\begin{tabular}{c c c c c c c}
\toprule
	\textbf{Distillation} & \textbf{Cora} & \textbf{Citeseer} & \textbf{Pubmed} \\
\midrule
	 EMI & 95.6 & 95.8 & 96.1 \\
	 EMI+Intra & 95.9 & 96.1 & 96.8\\
    EMI+inter & 96.5 & 96.8 & 97.1 \\
	EMI+Inter+Intra & 97.3 & 98.3 & 98.4 \\
\bottomrule
\end{tabular}
\end{table}

\begin{table}[!tp]
\caption{Ablation study of the distillation component on node clustering task, measured by Normalized Mutual Information (NMI) in \%.}
\label{stat:data:clus}
\centering
\footnotesize
\begin{tabular}{c c c c c c c}
\toprule
	\textbf{Distillation} & \textbf{Cora} & \textbf{Citeseer} & \textbf{Pubmed} \\
\midrule
	 EMI & 51.2 & 42.6 & 31.7 \\
	 EMI+Intra & 53.4 & 42.8 & 32.1\\
    EMI+inter & 55.6 & 43.5 & 32.8 \\
	EMI+Inter+Intra & 58.1 & 44.6 & 33.3 \\
\bottomrule
\end{tabular}
\end{table}

\end{document}